 \documentclass[12pt,onecolumn]{IEEEtran}
\makeatletter
\def\ps@headings{%
\def\@oddhead{\mbox{}\scriptsize\rightmark \hfil \thepage}%
\def\@evenhead{\scriptsize\thepage \hfil \leftmark\mbox{}}%
\def\@oddfoot{}%
\def\@evenfoot{}}
\makeatother
\IEEEoverridecommandlockouts
\usepackage{algorithm2e}
\usepackage{algorithmic}
\usepackage{float}
\usepackage[letterpaper]{geometry}
\usepackage{setspace}
\usepackage{graphicx}
\usepackage{amsmath,amssymb}
\usepackage{amsmath}
\usepackage{multirow}
\usepackage{verbatim}
\usepackage{epstopdf}
\usepackage{epsfig}
\usepackage{color}
\usepackage{subfig}
\usepackage{graphics}

\DeclareMathOperator{\argmax}{argmax}

\newtheorem{theorem}{Theorem}[section]
\newtheorem{remark}[theorem]{Remark}
\newtheorem{lemma}[theorem]{Lemma}

\newtheorem{proof}{Proof}

\ifodd 0
\newcommand{\rev}[1]{{\color{blue}#1}}
\newcommand{\com}[1]{\textbf{\color{red}(COMMENT: #1)}} 
\newcommand{\clar}[1]{\textbf{\color{green}(NEED CLARIFICATION: #1)}}

\newcommand{\response}[1]{\textbf{\color{magenta}(RESPONSE: #1)}} 
\else
\newcommand{\rev}[1]{#1}
\newcommand{\com}[1]{}
\newcommand{\clar}[1]{}
\newcommand{\response}[1]{}
\fi
\newcommand{\RNum}[1]{\uppercase\expandafter{\romannumeral #1\relax}}

\begin{document}
\title{Group Learning and Opinion Diffusion in a Broadcast Network}
\author{
\IEEEauthorblockN{Yang Liu, Mingyan Liu \\}
\IEEEauthorblockA{
Electrical Engineering and Computer Science\\
University of Michigan, Ann Arbor\\
\{youngliu, mingyan\}@umich.edu}
\thanks{This work is partially supported by the NSF under grants CIF-0910765 and CNS 1217689.}
}
\maketitle 

\begin{abstract}
We analyze the following group learning problem in the context of opinion diffusion: Consider a network with $M$ users, each facing $N$ options.  In a discrete time setting, at each time step, each user chooses $K$ out of the $N$ options, and receive randomly generated rewards, whose statistics depend on the options chosen as well as the user itself, and are unknown to the users.  
Each user aims to maximize their expected total rewards over a certain time horizon through an online learning process, i.e., a sequence of exploration (sampling the return of each option) and exploitation (selecting empirically good options) steps.  Different from a typical regret learning problem setting (also known as the class of multi-armed bandit problems), the group of users share information regarding their decisions and experiences in a broadcast network.  The challenge is that while it may be helpful to observe others' actions in one's own learning (i.e., second-hand learning), what is considered desirable option for one user may be undesirable for another (think of restaurant choices), and this difference in preference is in general unknown a priori.  Even when two users happen to have the same preference (e.g., they agree one option is better than the other), they may differ in their absolute valuation of each individual option.  

Within this context we consider two group learning scenarios, (1) users with uniform preferences and (2) users with diverse preferences, and examine how a user should construct its learning process to best extract information from other's decisions and experiences so as to maximize its own reward.  Performance is measured in {\em weak regret}, the difference between the user's total reward and the reward from a user-specific best single-action policy (i.e., always selecting the set of options generating the highest mean rewards for this user).   Within each scenario we also consider two cases: (i) when users exchange full information, meaning they share the actual rewards they obtained from their choices, and (ii) when users exchange limited information, e.g., only their choices but not rewards obtained from these choices. 
We show the gains from group learning compared to individual learning from one's own choices and experiences. 
\end{abstract} 

\section{Introduction}\label{sec:intro} 

We analyze the following group learning problem in the context of opinion diffusion: Consider a network with $M$ users, each facing $N$ options.  In a discrete time setting, at each time step, each user chooses $K$ out of the $N$ options, and receive randomly generated rewards, whose statistics depend on the options chosen as well as the user itself, and are unknown to the users.   
Each user aims to maximize its expected total reward over a certain time horizon through an online learning process, i.e., a sequence of exploration (sampling the return of each option) and exploitation (selecting empirically good options) steps.  Taken separately, an individual user's learning process may be mapped into a standard multi-armed bandit (MAB) problem which has been extensively studied, see e.g., \cite{LR85,Anantharam:M86/62,Auer:2002:FAM:599614.599677}.  

Our interest in this study, however, is on how an individual's learning process may be affected by ``second-hand learning'', i.e., by observing how others in the group act.  
The challenge is that while it may be helpful to observe others' actions to speed up one's own learning, what is considered desirable option for one may be undesirable for another (think of restaurant choices: one Yelp user's recommendation may or may not be useful for another), and this difference in preference is in general unknown a priori.  Moreover, even when two users happen to have the same preference (e.g., they agree one option is better than the other), they may differ in their absolute valuation of each individual option (again think of restaurant choices: two Yelp users may agree restaurant A is better than B, but one user may rate them 5 and 4 stars respectively, while the other may rate them 4 and 3 stars, respectively).  

Consequently it seems that if an individual wants to take others' actions into account in its own learning process, it would also need to figure out whether their preferences are aligned, which may add to the overhead in the learning process.  This raises the interesting question of whether learning from group behavior is indeed beneficial to an individual, and if so what type of learning algorithm can effectively utilize the group information in addition to its own direct observations.  This is what we aim to address in this paper. 

We will assume that users are heterogeneous in general, i.e., when using the same option they obtain rewards driven by different random processes with different mean values. 
We then consider two scenarios.  (1) In the first, users have {\em uniform preference ordering} of the $N$ options.  This means that even though they may value each options differently, they always prefer the same set of options.  (2) In the second, users have {\em diverse preference orderings} of the $N$ options, meaning that one user's best options are not so for another. 
Within each scenario we also consider two cases: (i) when users exchange full information, meaning they disclose  the actual rewards they obtained from their choices, and (ii) when users exchange limited information, e.g., only their choices but not rewards obtained from these choices. 
For each of these cases we examine how a user should construct its learning process to best extract information from other's decisions and experiences so as to maximize its own reward.  Performance is measured in {\em weak regret}, the difference between the user's total reward and the reward from a user-specific best single-action policy (i.e., always selecting the set of options generating the highest mean rewards for this user).   


This problem can also be viewed as a learning problem with contextual information (or side information in some literature), see e.g., \cite{Wang05banditproblems,LuPP10, DBLP:conf/nips/LangfordZ07}.  However, in these studies statistical information linking a user's own information and the side information is required in the following sense.  Denote by $X$ a user's observation and by $Y$ the side information (say shared information from other users), the knowledge of the conditional probability of observing $X$ (i.e., $p(X|Y)$) needs to be given or assumed.  In contrast, we do not require such statistical information; instead we examine how a user can estimate and learn from the shared, and possibly imperfect side information. 

The paper is organized as follows. Section \ref{model} gives the system model. Sections \ref{uniform} and \ref{diverse} analyze the uniform and diverse preference scenarios, respectively.  Numerical results are presented in Section \ref{simulation} and Section \ref{conclusion} concludes the paper. 
\section{Problem formulation and system model}\label{model}

Consider a system or network of $M$ users indexed by the set $\mathcal U = \{1,2,...,M\}$ and a set of available options denoted by $\Omega=\{1,2,...,N\}$. The system works in discrete time indexed by $t=1,2,...$. At each time step a user can choose up to $K$ options. 
For user $i$ an option $j$ generates {an IID} reward denoted by random variable $X^i_j$, with a mean reward given by $\mu^i_j := \mathbb E[X^i_j]$.  
We will assume that $\mu^i_l \neq \mu^i_j, l \neq j, \forall i \in \mathcal U$, i.e., different options present distinct values to a user.  
For simplicity of notations we will denote the set of top $K$ options (in terms of mean rewards) for user $i$ as $N^i_K$ and its complement $\overline{N}^i_K$.  Denote by $a^i(t)$ the set of choices made by user $i$ at time $t$; the sequence $\{a^i(t)\}_{t=1, 2, \cdots}$ constitutes user $i$'s policy. 

Following the classical regret learning literature, we will adopt the {\em weak regret} as a performance metric, which measures the gap between the total reward (up to some time $T$) of a given learning algorithm and the total reward of the best single-action policy given a priori average statistics, which in our case is the sum reward generated by the top $K$ options for a user.  This is formally given as follows for user $i$ adopting algorithm $a$: 
\begin{eqnarray}
{R^{i, a}(T) = T\cdot \sum_{j\in N^i_K} \mu^i_j - \mathbb E[\sum_{t=1}^T \sum_{j\in a^i(t)} X^i_j]} 
\end{eqnarray}
Goal of a learning algorithm is to minimize the above regret measure. 

{As mentioned in the introduction, we consider two scenarios. In the first case, users share the same preference ordering over the $N$ options, i.e., if $\mu^i_{j_1} > \mu^i_{j_2}$, $j_1, j_2 \in \Omega$, then $\mu^k_{j_1} > \mu^k_{j_2}$, $\forall k\neq i, k\in \mathcal U$. This implies that $N^i_K = N^k_K$, $\forall i, k\in \mathcal U$. 
This will be referred to as the {\em uniform preference} scenario.  

In the second, the {\em diverse preference} scenario, users have different preference orderings over the $N$ options. Specifically, in this case we will assume that the $M$ users may be classified into $G$ distinct groups,  indexed by the set $\mathcal G=\{1,2,...,G\}$, 
with users within the same group (say group $l$) having a unique $K$-preferred set $N^l_K$.  Note that even with the same preferred set, users may be further classified based on the actual ordering of these top $K$ options.  Our model essentially bundles these sub-classes into the same group, provided their top $K$ choices are the same.  This is because as a user is allowed $K$ choices at a time, further distinguishing their preferences within these $K$ options will not add to the performance of an algorithm. 
} 


Under each scenario, we further consider two types of information shared/exchanged by the users.  Under the first type, users disclose {\em full information}: they not only announce the decisions they made (the options they chose), but also the observations following the decisions, i.e., the actual rewards received from those options. Such announcements may be made at the end of each time step, or may be made periodically but at a lesser frequency.   The second type of exchange is {\em partial information} where users disclose only part of decisions and/or observations.  Specifically, we will assume that the users only share their decision information,  i.e., the set of choices they made, at the beginning of each time step, but withhold the actual observation/reward information following the decisions. 

\section{Group learning with uniform preference} \label{uniform}
Without loss of generality, we will assume that under the uniform preference ordering we have $\mu^i_1 > \mu^i_2 > ... > \mu^i_N, \forall i \in \mathcal U$. 
\subsection{Uniform preference, full information (U\_FULL)} 
\label{full}

This case will be referred to as U\_FULL. 
Under this model users not only broadcast their decisions within the network, but also release observations of selected options' quality/rewards at the end of each time step. 
Since users have the same preference ordering, a fact assumed to be known to the users, it would seem straightforward that one user could easily learn from another.  The challenge here lies in the fact that the statistics driving the rewards are not identical for all users even when using the same option.  So information obtained from another user may need to be treated differently from one's own observations.

In general the reward user $i$ obtains from option $j$ may be modeled as 
\begin{align}
X^i_j = f(X_j, \mathcal N_i,\mathcal L_i)~, 
\end{align}
where $f(\cdot)$ is some arbitrary unknown function, $X_j$ describes certain {\em intrinsic} or {\em objective} value of option $j$ that is independent of the specific user (e.g., the bandwidth of a channel, or the rating given to a restaurant by AAA, and so on), $\mathcal N_i$ is a noise term, and $\mathcal L_i$ captures user-specific features that affect the {\em perceived} value of this option to user $i$ (e.g., user $i$'s location information or transceiver specification which may affect its perceived channel quality, or user $i$'s dietary origin which may affect its preference for different types of restaurants). 
For simplicity in this study we will limit our attention to the following special case of user-specific valuations, where the rewards received by two users from the same option are given by a linear relationship:  
\begin{align}
{\mu^i_j/\mu^k_j = \delta^{i,k}_j~.  } 
\end{align}
The scaling factor $\delta^{i,k}_j$ 
will be referred to as the {\em distortion} or {\em distortion factor} between two users.  

{Under this model it can be seen that a user could recover/convert observations from other users for its own use by estimating the distortion. 
Consider two users $i$ and $k$, and option $j$. Denote by ${r}_j^i(t)$ the sample mean reward collected by $i$ directly itself from option $j$ up to time $t$.  This quantity is not only available to user $i$, but also to all other users $k$ due to the full information disclosure, and vice versa.  User $i$ then estimates the distortion between itself and user $k$ by $\tilde\delta^{i,k}_j(t) = r^i_j(t)/r^k_j(t)$. } 


With this quantity we then make the following simple modification to the well-known UCB algorithm introduced in \cite{Auer:2002:FAM:599614.599677}. In the original UCB (or rather, a trivial multiple-play extension of it), user $i$'s decision $a^i(t)$ at time $t$ is entirely based on its own observations.  Specifically, denote by $n^i_j(t)$ the number of times user $i$ has selected option $j$ up to time $t$. 
The original UCB then selects option $j$ at time $t$, if its index value given below is among the $K$ highest: 
\begin{eqnarray}
\mbox{UCB index:} ~~~~ r^i_{j}(t) + \sqrt{\frac{2 \log t}{n^i_j(t)}}~. 
\end{eqnarray}
Under the modified algorithm (referred to as the U\_FULL algorithm), 
option $j$ is selected at time $t$ if its index value defined below is among the $K$ highest: 
\begin{eqnarray}
\rev{\mbox{U\_FULL index:} ~~~~ \frac{\sum_{m=1}^t X^i_j(m)\cdot \mathbf{I}_{j\in a^i(m)} +  \sum_{k\neq i} \sum_{m=1}^t \tilde\delta^{i,k}_j(m) X^k_j(m)\cdot \mathbf{I}_{j\in a^k(m)}}{\sum_{i \in\mathcal U}n^i_j(t)} + \sqrt{\frac{2 \log t}{\sum_{i \in\mathcal U}n^i_j(t)}}. }
\end{eqnarray}

We have the following results on algorithm U\_FULL. 
\begin{lemma} $\forall \epsilon >0$,
\begin{equation}
P(|\tilde\delta^{i,k}_j(t)-\delta^{i,k}_j|>\epsilon) \leq 1/t^{d_U}
\end{equation}
with $d_U$ being some finite positive constant. \label{uFullLma}
\end{lemma}
\begin{proof}
Proof can be found in Appendix-\ref{uFullDis}.
\end{proof}
\begin{theorem}
The weak regret of {user $i$} under U\_FULL is upper bounded by 
\begin{align}
R_{\text{U\_FULL}}^i(t) \leq \sum_{j \in \overline{N}^i_K}\lceil\frac{8 \log t}{M\cdot \Delta^i_{j}} \rceil+\text{const.} ~.  
\end{align}
where $\Delta^i_{j} = \mu^i_K - \mu^i_j$.\label{uFullThm}
\end{theorem}
\begin{proof}
Proof can be found in Appendix-\ref{uFull}. 
\end{proof}
Under the original UCB algorithm \cite{Auer:2002:FAM:599614.599677} a single user's weak regret is upper bounded by (the superscript $i$ is suppressed here because the result applies to any single user) 
\begin{align}
R_{\text{UCB}}(t) \leq \sum_{j \in \overline{N}_K}\lceil\frac{8 \log t}{\Delta^i_{j}}\rceil + \text{const.} 
\end{align}
Therefore we see that there is potential gain in group learning. Note however the improvement is not guaranteed as it appears in an upper bound, which does not necessarily imply better performance. The performance comparison is shown later via simulation. 

It can be shown that similar result exists when the full information is broadcast at periodic intervals but not necessarily at the end of each time step.

\subsection{Uniform preference, partial information (U\_PART)}\label{partial}


We now consider the case where users only share their decisions/actions, but not their direct observations.  This case (and the associated algorithm) will be referred to as U\_PART.  
The difficulty in this case comes from the fact that to a user $i$, even though other users' actions reflect an option's relative value to them (and by positive association to user $i$ itself), the actions do not directly reveal the actual obserations. 

Denote by $n_{j}(t)$ the total number of times option $j$ has been selected by the entire group up to time $t$. Then $\beta_j(t):= \frac{n_{j}(t)}{\sum_{l \in \Omega} n_{l}(t)}$ denotes the frequency at which option $j$ is being used by the group up to time $t$.  This will be referred to as the group recommendation or behavior. 
%
Several observations immediately follow.  Firstly, we have 
$
\sum_{j \in \Omega} \beta_{j}(t) = 1, \forall t.
$
Secondly, as time goes on, we would like better options $j$ to increasingly correspond to larger $\beta_j(t)$. 

With these observations, we construct the following algorithm U\_PART, by biasing toward potentially good options as indicated by the group behavior.  
%
Under the U\_PART algorithm, option $j$ is selected at time $t$ if its index value defined below is among the $K$ highest: 
\begin{eqnarray}
\mbox{U\_PART index:} ~~~~ r^i_j(t) - \alpha(1-\beta_j(t))\sqrt{\frac{\log{t}}{t}} + \sqrt{\frac{2 \log t}{n^i_j(t)}}~, 
\end{eqnarray}
where $\alpha$ is a weighting factor over the group recommendation.  

A few remarks are in order.  In the above index expression, the middle, bias term serves as a penalty: a larger group frequency $\beta_j(t)$ means a smaller penalty. But its effect diminishes as $t$ increases.  This reflects the notion that as time goes on a user becomes increasingly more confident in its own observations and relies less and less on the group recommendation. 
Lastly, the weight factor $\alpha$ captures how much the user values the group recommendation compared to its own observations, with a small value indicating a small weight. 

We have the following result on the U\_PART algorithm.
\begin{theorem}
The weak regret of user $i$ under U\_PART is upper bounded by 
\begin{align}
R_{\text{U\_PART}}^i(t) \leq \sum_{j \in \overline{N}^i_K}\lceil\frac{(6+\epsilon) \log t}{\Delta^i_{j}}\rceil + \text{const.} ~, 
\end{align}
where $\epsilon$ is some arbitrarily small positive number. \label{uniform:partial}
\end{theorem}
\begin{proof}
Proof can be found in Appendix-\ref{uPart}. 
\end{proof}

Again, compared to the bounds from the original UCB, we potentially achieve a better performance as the bound constant decreases from $8$ to $6+\epsilon$ with the group recommendation mechanism, but with the same cautionary note on the upper bound.  The performance comparison is shown in simulation results later. 

\section{Group learning with diverse preferences}\label{diverse}
In this part we consider a more complicated case. Suppose users within a group have different tastes over options and we divide the group of users into multiple sub-groups based on their preferences. Different groups have different preference order over options, i.e., the assumption  $\mu^i_1 > \mu^i_2 > ... > \mu^i_N, \forall i \in \mathcal U$ would not hold necessarily. 

Specifically, all $M$ users are divided into $G$ groups $\mathcal G=\{1,2,...,G\}$. Users within same group share same preferences over options; while users from different groups have different ones. We further assume the set of top $K$ options differs from group to group. Therefore all together we have $C^K_N$ different group preferences and $G \leq C^K_N$.


\subsection{Diverse preferences, full information (D\_FULL)}

 We again estimate the pair-wise distortion factor in a manner similar to the uniform preference case:  $\tilde\delta^{i,k}_j(t) = r^i_j(t)/\tilde r^k_j(t)$.
 The resulting D\_FULL algorithm run by user $i$ then selects, at time $t$, an option $j$ if its index value given below is among the $K$ highest: 
\begin{eqnarray}
\rev{\mbox{D\_FULL index:} ~~~~ \frac{\sum_{m=1}^t X^i_j(m)\cdot \mathbf{I}_{j\in a^i(m)} +  \sum_{k\neq i} \sum_{m=1}^t \tilde\delta^{i,k}_j(m) X^k_j(m)\cdot \mathbf{I}_{j\in a^k(m)}}{\sum_{i \in \mathcal U} n^i_j(t)} + \sqrt{\frac{2 \log t}{\sum_{i \in \mathcal U} n^i_j(t)}}. }
\end{eqnarray}
We have the following result on algorithm D\_FULL. 
\begin{lemma} $\forall \epsilon >0$,
\begin{equation}
P(|\tilde\delta^{i,k}_j(t)-\delta^{i,k}_j|>\epsilon) \leq 1/t^{d_D}
\end{equation}
with $d_D$ being some finite positive constant. \label{uFullLma}
\end{lemma}
\begin{theorem}
Under algorithm D\_FULL, user $i$'s weak regret is upper bounded by 
\begin{align}
R^i_{\text{D\_FULL}}(t) \leq \sum_{j \in \overline{N}^i_K}\lceil\frac{8 \log t}{M\cdot \Delta^i_{j}}\rceil +\text{const.} ~, 
\end{align}
\end{theorem}
\begin{proof}
The proof follows similarly as in the uniform case and the details are thus omitted. 
\end{proof}



\subsection{Diverse preferences, partial information (D\_PART)}


As in the case of D\_FULL we can track for each user $n^i_j(t)$ and obtain the frequency of choices $\beta^i_j(t)$, and use this information to perform group classification.  \rev{A user $i$} 
then assigns a weight to an option $j$ given by the ratio 
With diverse group preferences, the direct or raw sample frequency implies nothing. For example, group 1 observes option 1 ten times while only have chosen option 2 once; group 2 picked option 1 for only once while stick with option 2 ten times. It is obvious that group 1 prefers option 1 over option 2 while group 2 prefers option 2. However if users from each group use the globally observed frequency, they will assign option 1 and 2 with equal weight $\frac{10+1}{11+11} = 1/2$ which will make the extra group information useless. Thus we need a new mechanism to distinguish different groups' preferences.

We introduce the following sample frequency based group identity classification mechanism. Each user keeps the same set of statistics $n^i_j(t)$ as before: the number of times user $i$ is seen using option $j$.  From these a user tries to estimate another's preference by ordering the statistics: at time $t$ user $i$'s preference is estimated to be the set $\tilde{N}^i_K(t)$, which contains elements/options $j$ whose frequency $n^i_{j}(t)$ is among the $K$ highest of all $i$'s frequencies. 
User $i$ is then put in the preference group $l$ with whose (known) preferred set $N^l_K$ its estimated preference $\tilde{N}^i_K(t)$ is the closest in distance, defined as follows: 
\begin{align}
\mbox{Assign user $i$ to group $l^*$ if:} ~~~ l^* = \argmax_{l\in \mathcal{G}} D^{i,l} (t) = 
|\tilde{N}^i_K(t) \cap N^l_K | ~, 
\end{align}
with ties broken randomly. 

%
Our algorithm proceeds in parallel with the uniform group case except for the following difference: each group will assign another group with certain discount for their observations instead of raw statistics. To be specific, user $i$ will assign the weight to the $m$th option in the following sense
\begin{align}
\beta^i_j (t) = \frac{\sum_{k \in \mathcal U}(n^{k}_j(t))^{\omega^{i,k}}}{\sum_{m \in \Omega}\sum_{k \in \mathcal U}(n^{k}_m(t))^{\omega^{i,k}}}~, 
\end{align} 
where weights $\omega^{i,k} = 1$ if $i$ estimates user $k$ to be in the same group as itself, and $\omega^{i,k} < 1$ otherwise. \rev{$\omega^{i,k}$ can also be chosen as a function of the group distance. } 
\com{I added the superscripts $k$ in the above expression, which I think was missing...} 

The resulting algorithm D\_PART is as follows, where user $i$ chooses option $j$ if its index value is among the top $K$ highest: 
\begin{eqnarray}
\mbox{D\_PART index:} ~~~~ r^i_j(t) - \alpha(1-\beta^i_j(t))\sqrt{\frac{\log{t}}{t}} + \sqrt{\frac{2 \log t}{n^i_j(t)}}~, 
\end{eqnarray}

\begin{theorem}
For each user $j$ associating with group $r$ we have the probability of incorrect classification at time $t$ is bounded as
\begin{align}
P(g_j(t) \neq r)  \leq C_1 \cdot \frac{\log t}{t}, \forall (j,r),t.
\end{align}
for some positive constant $C_1$.\label{dPartial}
\end{theorem}
\begin{proof}
Proof can be found in Appendix-\ref{dPart}. 
\end{proof}
The upper bound on the weak regret under D\_PART is the same except for a different constant, as in the case of U\_PART. 
\begin{theorem}
Under algorithm D\_PART, user $i$'s weak regret is upper bounded by 
\begin{align}
R^i_{\text{D\_PART}}(t) \leq \sum_{j \in \overline{N}^i_K}\lceil\frac{(6+\epsilon) \log t}{\Delta^i_{j}}\rceil + \text{const.}, 
\end{align} 
\end{theorem}
\begin{proof}
The proof follows similarly as in the uniform case with partial information exchange with the following explanation.  Notice the last step to establish the $6+\epsilon$ bound is to prove
\begin{align}
\frac{n_{j}(t)}{\sum_{i \in \Omega} n_{i}(t)} \rightarrow 0, \forall j \in \overline{N}_K.
\end{align}
Then similarly here we need to establish
\begin{align}
\frac{\sum_{j \in \mathcal U}n^{\omega^{g_i,g_j}}_m(t)}{\sum_{l \in \Omega}\sum_{j \in \mathcal U}n^{\omega^{g_i,g_j}}_l(t)} \rightarrow 0, \forall m \in \overline{N}^i_K.
\end{align}
To show this, first of all we notice
\begin{align}
\sum_{l \in \Omega}\sum_{j \in \mathcal U}n^{\omega^{g_i,g_j}}_l(t) = \mathcal O(t)
\end{align}
Then for user $j$ within the same group as user $i$ we have
\begin{align}
E[n^{\omega^{g_i,g_j}}_m(t)] = E[n^{j}_m(t)] = \mathcal O(\log t)
\end{align}
For other user $j$ from different groups we know 
\begin{align}
E[n^{\omega^{g_i,g_j}}_m(t)] < E[n^{j}_m(t)] \leq \mathcal O(\log t)
\end{align}
since $\omega^{g_i,g_j} < 1$. Meanwhile the chance of mis-classifying a user from a different group to a same group is upper bounded by $\mathcal O(\frac{\log t}{t})$, and the number of mis-drift is at most given by 
\begin{align}
\mathcal O(t)\cdot \mathcal O(\frac{\log t}{t}) = \mathcal O(\log t)
\end{align} 
which helps us establish 
$
\frac{\sum_{j \in \mathcal U}n^{\omega^{g_i,g_j}}_m(t)}{\sum_{l \in \Omega}\sum_{j \in \mathcal U}n^{\omega^{g_i,g_j}}_l(t)} \rightarrow 0, \forall m \in \overline{N}^i_K.
$

\end{proof}

\section{Numerical results}\label{simulation}


We start with U\_FULL. In our simulation we have three users with five independent options; each user targets the top three options at each time, i.e., $M=K=3, N=5$. Furthermore the five options' reward statistics are given by exponentially distributed random variables. 
The distortion factor at each user for each option is modeled as a Gaussian random variable with certain mean and variance 1. 
\begin{figure}[h!]
\centering
\includegraphics[width=0.6\textwidth,height=0.3\textwidth]{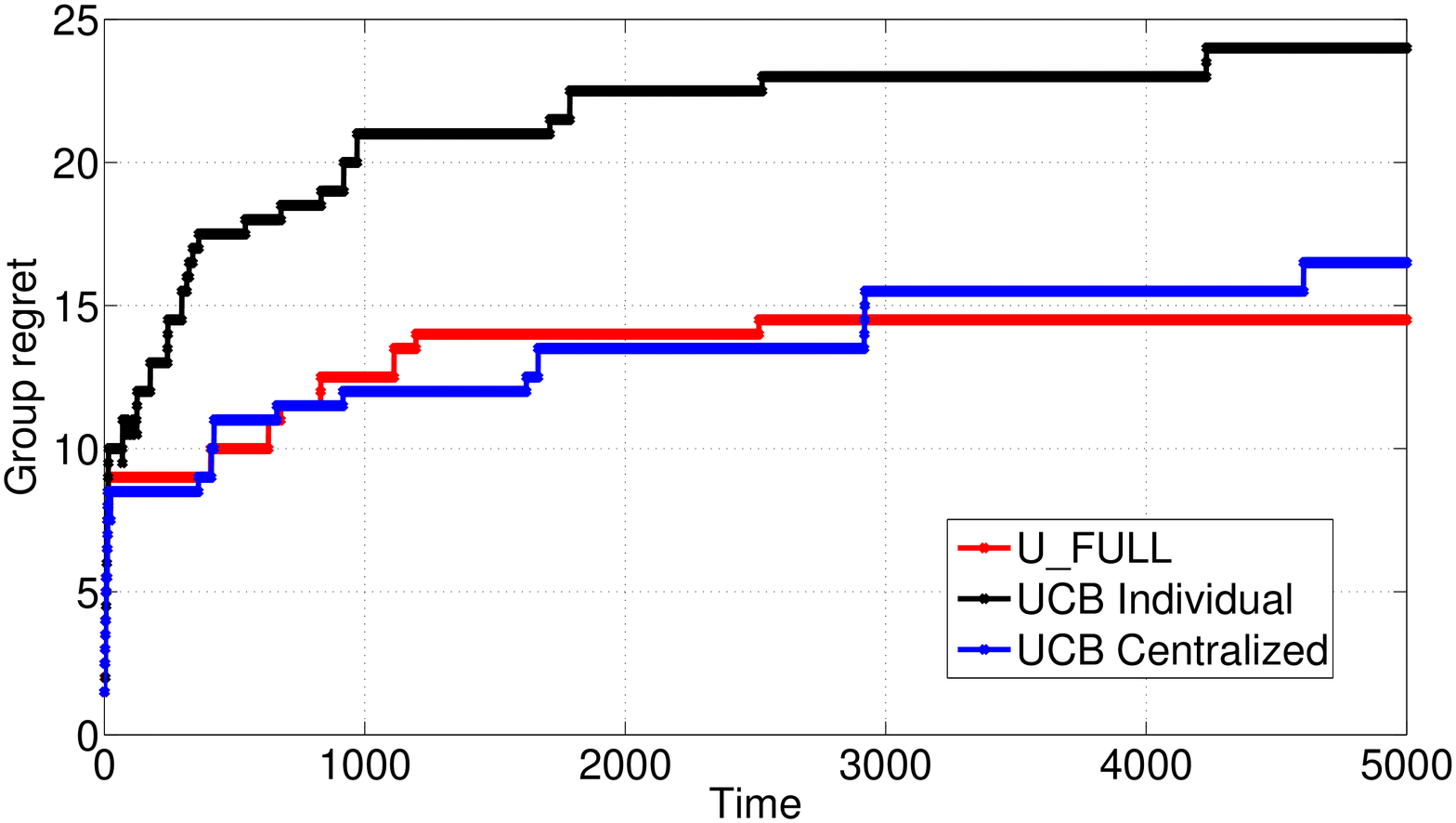}
\caption{Performance comparison of U\_FULL and UCB individual.}\label{finite}
\end{figure}

From Fig. \ref{finite} we see with full information exchange the system's performance can be greatly improved compared with individual learning; moreover, its performance is comparable with a centralized scheme (denoted UCB Centralized in the figure), whereby the $M$ users are centrally controlled and coordinated in their learning using UCB, and allowing simultaneous selection of the same options by multiple users.

Next we show the performance of U\_PART. The simulation setting is the same as the one above and is not repeated here.  
\begin{figure}[h!]
\centering
\subfloat[Performance comparison]{\label{gl}\includegraphics[width=0.5\textwidth,height=0.3\textwidth]{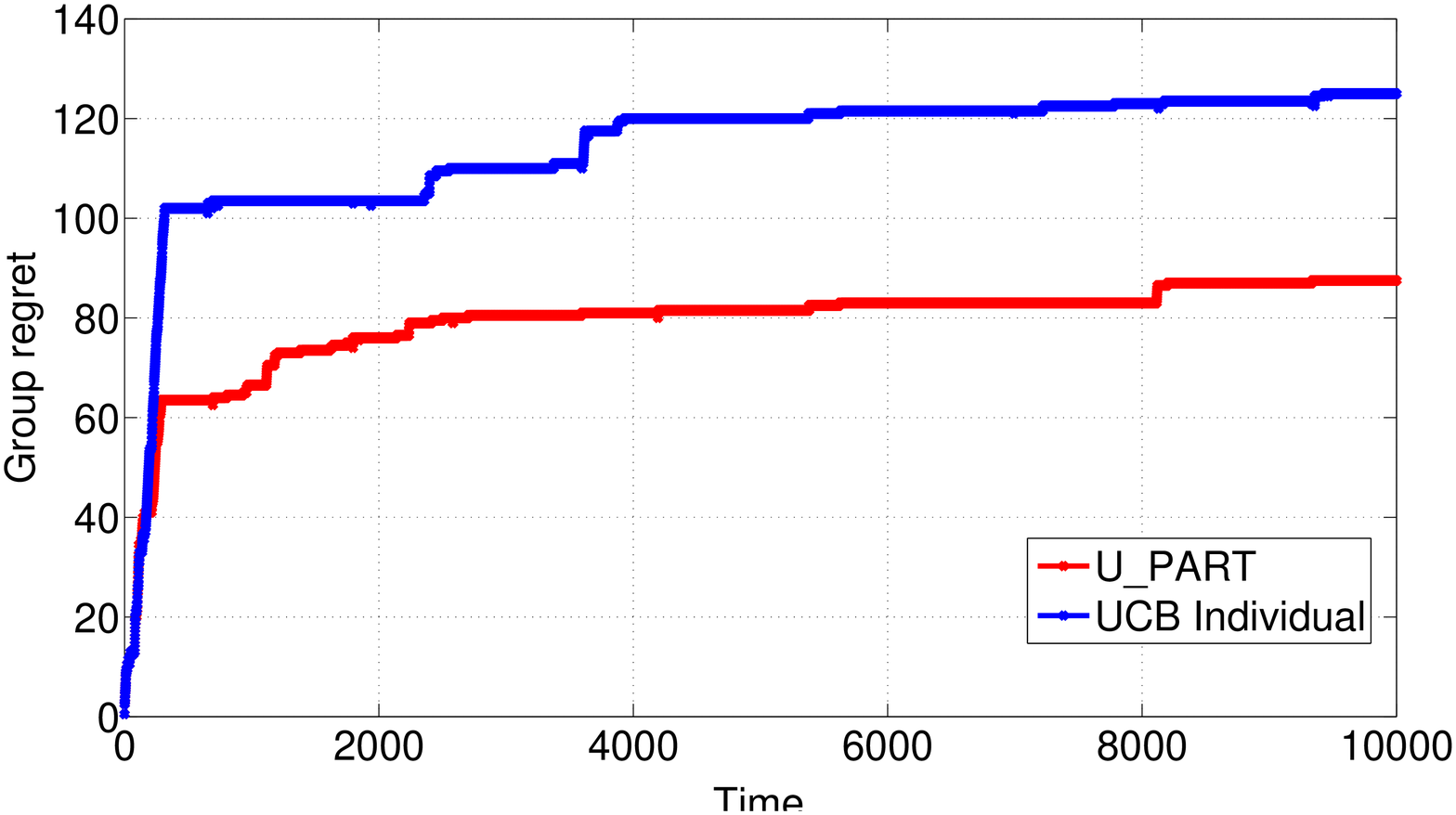}} 
\subfloat[Different $\alpha$]{\label{al1}\includegraphics[width=0.5\textwidth,height=0.3\textwidth]{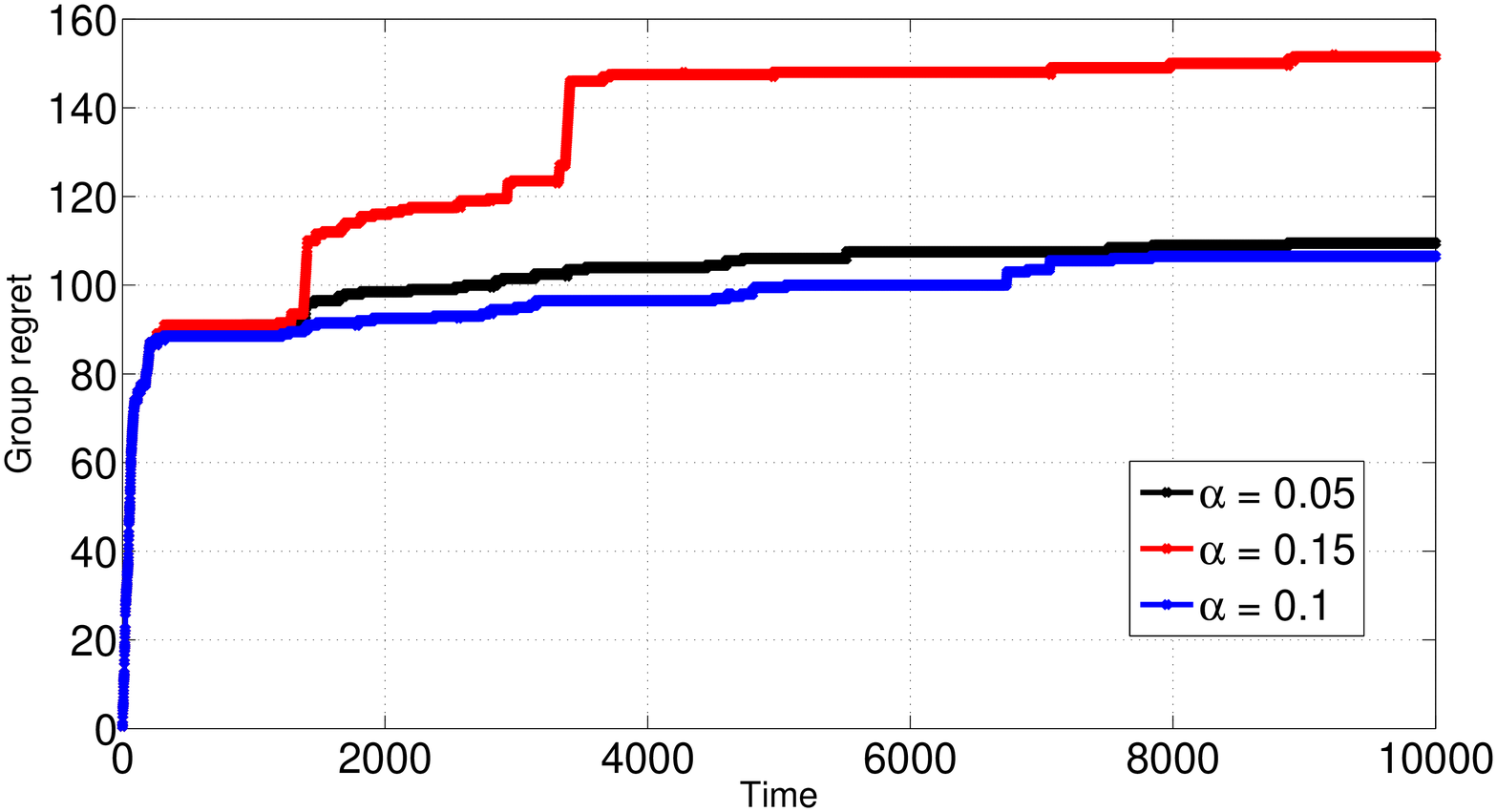}}
\caption{Performance comparison of U\_PART and UCB Individual.}
\end{figure}
From Fig. \ref{gl} we see that U\_PART outperforms multiuser UCB with individual learning.  
We also see from Fig. \ref{al1} that though a larger $\alpha$ results in a larger upper bound, the actual performance does not necessarily increase with $\alpha$.

We end this section by simulating a network with diverse group preferences.  As we mentioned in previous sections, the major difference between learning algorithms of diverse preferences and uniform preference is each user estimates other users' group identity before taking actions over observed/reported samples. Therefore instead of presenting similar regrets results as in the previous cases, we present the mis-classification rate of our algorithm, given in Fig. \ref{err_rate}. 

\begin{figure}[h!]
\centering
\label{alpha}\includegraphics[width=0.6\textwidth,height=0.3\textwidth]{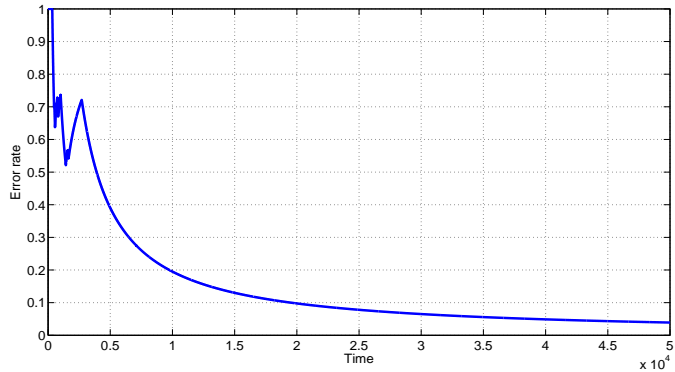}
\caption{Performance of error rate}\label{err_rate}
\end{figure}

\section{Conclusion}\label{conclusion}

In this paper we considered group learning problem in the context of opinion diffusion and analyzed two scenarios: uniform group preference vs. diverse group preferences.  For each case we also considered sharing full vs. partial information, and constructed UCB-like index based group learning algorithms and derived their associated upper bounds on weak regret.  These upper bounds are in general better than the original upper bound obtained by UCB when a user learns in isolation.  This points to the potential gain by combining first-hand and second hand learning.


\bibliographystyle{plain}
\bibliography{myref}
\appendices
\section{Proof of Lemma \ref{uFullLma}:} 
\label{uFullDis}
For simplicity of presentation, we omit all sub and super-scripts in this proof when there is no confusion. 
\begin{align}
&P(|\delta(t)-\delta^{*}|>\epsilon) = P(\delta(t)>\delta^{*}+\epsilon) + P(\delta(t)<\delta^{*}-\epsilon)
\end{align}
Let's consider $P(\delta(t)>\delta^{*}+\epsilon)$ and denote $c = \delta^{*}+\epsilon$.
\begin{align}
P(\delta(t)>c) &= P(\frac{r^i_j(t)}{r^k_j(t)}>c) = P(r^i_j(t) > c \cdot r^k_j(t))
\end{align}

and
\begin{align}
&P(r^i_j(t) > c \cdot r^k_j(t)) = \int_x P(r^i_j(t) > c\cdot x)\cdot P(r^k_j(t)= x) dx \nonumber \\
&=\int_{x \leq \mu^k_j-\epsilon} P(r^i_j(t) > c\cdot x)P(r^k_j(t)= x) dx \nonumber \\
&+ \int_{\mu^k_j-\epsilon < x < \mu^k_j+\epsilon}P(r^i_j(t) > c\cdot x)\cdot P(r^k_j(t)= x) dx\nonumber \\
&+\int_{x \geq \mu^k_j+\epsilon} P(r^i_j(t) > c\cdot x)\cdot P(r^k_j(t)= x) dx
\end{align}
Next we first turn to the analysis of hitting times of each option $T^i_j(t)$ and $T^k_j(t)$.
\begin{lemma}
For $j \in N_K$, 
\begin{align}
E[T^i_j(t)] = \mathcal O(t), E[T^k_j(t)] = \mathcal O(t)
\end{align}
and for $j \in \overline{N}_K$, 
\begin{align}
E[T^i_j(t)] = \mathcal O(\log t) , E[T^k_j(t)]& = \mathcal O(\log t)
\end{align}
\end{lemma}
\begin{proof} The complete proof can be obtained by following Lai(\cite{LR85}) and Zhao(\cite{1104491})'s work. However this is not the emphasize of our paper and requires lots of repetition from previous works; thus we only sketch the basic proof and idea here.

Define the corresponding term under centralized and decentralized system as
$$
E_{\text{centr.}}[T^i_j(t)],E_{\text{decen.}}[T^k_j(t)], E_{\text{centr.}}[T^i_j(t)],E_{\text{decen.}}[T^{k}_j(t)]
$$
and we have
\begin{align}
E_{\text{centr.}}[T^i_j(t)] \leq E[T^i_j(t)] \leq E_{\text{decen.}}[T^i_j(t)]
\end{align}
\begin{align}
E_{\text{centr.}}[T^{k}_j(t)] \leq E[T^{k}_j(t)] \leq E_{\text{decen.}}[T^{k}_j(t)]
\end{align}
The reason for above two inequalities is that the available information for each model making decision follows as
\begin{align}
I_{\text{decen.}} \subseteq I_{\text{group}} \subseteq I_{\text{centr.}}
\end{align}
And we know from previous works $\forall j \in N_K$
\begin{align}
E_{\text{centr.}}[T^i_j(t)] = \mathcal O(t), E_{\text{decen.}}[T^k_j(t)] &= \mathcal O(t)
\end{align}
and  $\forall j \in \overline{N}_K$
\begin{align}
E_{\text{centr.}}[T^i_j(t)] = \mathcal O(\log t), E_{\text{decen.}}[T^k_j(t)] &= \mathcal O(\log t)
\end{align}
\end{proof}
\begin{remark}
For our modified UCB we are sure for $j \in \overline{N}_K$, $E[T_j(t)] \geq \mathcal O(\log t)$ since it is not necessarily the optimal policy; and we will show later in the next proof we indeed we have $E[T_j(t)] = \mathcal O(\log t)$ and $E[T_j(t)] = \mathcal O(t)$ for $j \in N_K$. 
\end{remark}

Consider $x = \mu^k_j \pm \epsilon$. In this case
$
P(r^k_j(t) = x) \rightarrow 1
$. 
For other $x$, we have
$
P(r^k_j(t) = x) \leq e^{-2\epsilon^2 T^{k}_j(t)}
$. 
Therefore we have
\begin{align}
&\int_{x \leq \mu^k_j-\epsilon} P(r^i_j(t) > c\cdot x)\cdot P(r^k_j(t) = x) dx \nonumber \\
&\leq \int_{x \leq \mu^k_j-\epsilon} P(r^i_j(t) > c\cdot x)\cdot e^{-2\epsilon^2 T^k_j(t)} dx \nonumber  \\
&\leq e^{-2\epsilon^2 T^k_j(t)} \int_{x \leq \mu^k_j-\epsilon} 1 dx =\mathcal O(e^{-T^k_j(t)}) \label{term1}
\end{align}
and 
\begin{align}
& \int_{\mu^k_j-\epsilon \leq x \leq \mu^k_j+\epsilon} P(r^i_j(t) > c\cdot x)\cdot P(r^k_j(t) = x) dx \nonumber \\
&\leq \int_{\mu^k_j-\epsilon \leq x \leq \mu^k_j+\epsilon}P(r^i_j(t) > c\cdot x) dx \nonumber \\
&\leq \int_{\mu^k_j-\epsilon \leq x \leq \mu^k_j+\epsilon} e^{-2\epsilon^2 T^i_j(t)} dx =\mathcal O(e^{-T^i_j(t)})
\end{align}
\begin{align}
&\int_{x \geq \mu^k_j+\epsilon}  P(r^i_j(t) > c\cdot x)\cdot P(r^k_j(t) = x) dx \nonumber \\
&\leq e^{-2\epsilon^2 T^k_j(t)}\int_{x \geq \mu^k_j+\epsilon} P(r^i_j(t) > c\cdot x) dx \nonumber \\
&\leq e^{-2\epsilon^2 T^k_j(t)}\int_{x \geq \mu^k_j+\epsilon} e^{-T^k_{j}(t)x^2\epsilon} dx = \mathcal O(e^{-(T^i_j(t)+T^k_j(t))})
\end{align}
Now we investigate the expectation of each term as following. Take (\ref{term1}) for example
\begin{align}
\mathcal O (E[e^{-T^k_j(t)} ] ) < \mathcal O (e^{-(e^{-1}+1)\cdot E[T^k_j(t)]}  )
\end{align}
As $E[T^k_j(t)] \geq \mathcal O(\log t)$ ($E[T^k_j(t)] = \mathcal O(\log t)$ or $E[T^k_j(t)] = \mathcal O(t)$ depending on $j$.), we have
\begin{align}
\mathcal O (E[e^{-T^k_j(t)} ] ) <  \mathcal O (e^{-(e^{-1}+1)\cdot \log t}  ) \leq 1/t^d
\end{align}
here $d$ is some finite positive number. Other terms can be similarly analyzed and we proved the lemma.


\section{ Proof of Theorem \ref{uFullThm}:}\label{uFull}

We follow Auer.'s notation in \cite{Auer:2002:FAM:599614.599677} where UCB is first introduced and proved; and we try to bound the number of sub-optimal arms that are played. In order to make the proof smooth, we made another assumption here : each user only takes at most the same amount of samples (regarding order) from other users for the each option. Consider such an option $j$ of user $i$. 
\begin{align}
T^i_j(t) &\leq \l+\sum_{n=l+1}^t\{I^i_n = j,T^i_j(n-1) \geq \l\} \nonumber \\
&\leq \l + \sum_{n=l+1}^t \{\min_{0<s^i<n}r^{i,*}_{s^i} + c^i_{n-1,s^i} \leq \max_{\l \leq s^i_j < n}r^i_{j,s^i_j} + c^i_{n-1,s^i_j} \} \nonumber \\
&\leq \l + \sum_{n=l+1}^t \{\min_{0<s^i<n}r^{i,*}_{s^i} + c^i_{n-1,s^i} \leq \max_{\l \leq s^i_j < n}r^i_{j,s^i_j} + c^i_{n-1,s^i_j} \} \nonumber \\
&\leq \l + \sum_{n=1}^{\infty} \sum_{s^i=1}^{n-1}\sum_{s^i_j = \l}^{n-1}\{r^{*,i}_s + c^i_{n,s^i} \leq r^{i}_{j,s^i_j}+c_{n,s^i_j}\}
\end{align}
Observe $r^{i,*}_{s^i} + c^i_{n,s^i} \leq r^i_{j,s^i_j}+c_{n,s^i_j}$ implies that at least one of the following must hold,
\begin{align}
r^{i,*}_{s^i} &\leq \mu^{i,*}-c^i_{n,s^i}, r^i_{j,s^i_j} \geq \mu^i_j + c_{n,s^i_j}, \mu^{i,*} < \mu^i_j + 2c^i_{n,s^i_j}
\end{align}
We bound each term. First of all we want to know the error of calculating the sample means dues to imperfect distortion recovery. We have the number of errors upto time $t$ bounded by (take user $k$ for example)
\begin{align}
\sum_{n=1}^{s^k_j} \frac{1}{n^{d_U}}\leq \frac{(s^k_j)^{1-d_U}-1}{1-d_U} 
\end{align}
Denote the largest deviation by $\delta_{\max}$ we have the distortion factor in sample mean bounded as 
\begin{align}
C\cdot \frac{(s^k_j)^{1-d_U}}{\sum_{u \in \mathcal U} s^u_j}\end{align}
We now show
\begin{align}
\sqrt{\frac{\log n}{\sum_{u \in \mathcal U} s^u_j}} > C \cdot \frac{(s^k_j)^{1-d_U}}{\sum_{u \in \mathcal U} s^u_j}\label{bd}
\end{align}
For $s^i_j = \mathcal O(\log n)$, since $d_U > 0$ we have
\begin{align}
(\sum_{u \in \mathcal U} s^u_j)^{1/2} > C\cdot \frac{ (s^k_j)^{1-d_U}}{\sqrt{\log n}}
\end{align} 
If $s^i_j > \mathcal O(\log n)$, consider two cases. If $s^k_j = \mathcal O(\log n)$, above holds obviously. If $s^k_j > \mathcal O(\log n)$, through the proof of Lemma \ref{uFullLma} we know $d_U \geq 1$ (details omitted) and again we have (\ref{bd}) holds. Therefore we have 
\begin{align}
&P\{r^{i,*}_{s^i} \leq \mu^{i,*}-\sqrt{2}\sqrt{\frac{\log n}{\sum_{u \in \mathcal U} s^u_j}} \pm \frac{C'\cdot \sum_{k \in \mathcal U}(s^k)_j^{1-d_U}}{\sum_{u \in \mathcal U} s^u_j}\} \nonumber \\
&\approx P\{r^{i,*,c}_{s^i} \leq \mu^{i,*}-\sqrt{2}\sqrt{\frac{\log n}{\sum_{u \in \mathcal U} s^u_j}}\} \nonumber \\
&\leq e^{-4\log n} = n^{-4 }
\end{align}
And similarly
\begin{align}
P\{r^i_{j,s^i_j} &\geq \mu^i_j + c^i_{n,s^i_j} \} \leq n^{-4}
\end{align}
For the last term
\begin{align}
\mu^{i,*} &< \mu^i_j + 2c^i_{n,s^i_j}
\end{align}
Let $\l = \lceil \frac{8}{M} \log t/(\Delta^i_j)^2 \rceil$
\begin{align}
&\mu^{i,*}-\mu^i_j-2c^i_{n,s^i_j} \nonumber \\
&\geq \mu^{i,*}-\mu^i_j-2\sqrt{2}\sqrt{\log t/\sum_{u \in \mathcal U} s^u_j} \nonumber \\
&\geq \mu^{i,*}-\mu^i_j-\Delta^i_j = 0
\end{align}
The rest of the proof follows \cite{Auer:2002:FAM:599614.599677} and thus omitted. 

The only thing left to be proved is the number of optimal arms are sensed with order $\mathcal O(t)$ which is needed in proving Lemma \ref{uFullLma}. Suppose there is an optimal option failed to be sensed at order $\mathcal O(t)$, then there must be a sub-optimal arm $j \in \overline{N}_K$ being sensed with order $\mathcal O(t)$. Since we have $E[T^i_j(t)] \leq \sum_{j \in \overline{N}^i_K}\lceil\frac{8 \log t}{M\cdot (\Delta^i_{j})^2}\rceil$, we know the only possible case is with $d_U = 0$ and from the proof of Lemma \ref{uFullLma} we know the number of optimal arms must be bounded as a constant. Therefore check back the UCB index. For the optimal arm we have $\sqrt{\frac{\log t}{\mathcal O(1)}}$ as the bias term in UCB index while for $j$ it is $\sqrt{\frac{\log t}{\mathcal O( t)}}$ we know with a large enough $t$, the index of the optimal arm must be larger than $j$ which contradicts the fact that the optimal arm is only sensed with constant numbers; and we thus proved the theorem. 


\section{ Proof of Theorem \ref{uniform:partial}:}\label{uPart}

In the following literature for simplicity we denote
\begin{align}
c^i_{t,j} &= \sqrt{\frac{2 \log t}{\sum_{u \in \mathcal U}n^u_j(t)}}, \hat{c}^i_{t,j} = \sqrt{\frac{2 \log t}{\sum_{u \in \mathcal U}n^u_j(t)}} - \alpha [1-\beta_{j}(t)]\sqrt{\frac{\log t}{t}}
\end{align}
By following the same trick from Auer.'s work \cite{Auer:2002:FAM:599614.599677}, to bound the regret we need to bound the number of sub-optimal arms that are played. Suppose $j \in \overline{N}_K$
\begin{align}
T^i_j(t) &\leq \l+\sum_{n=\l+1}^t\{I^i_n = j,T^i_j(n-1) \geq \l\} \nonumber \\
&\leq \l + \sum_{n=\l+1}^t \{\min_{0<s^i<n}r^{i,*}_{s^i} + \hat{c}^i_{n-1,s^i} \leq \max_{\l \leq s^i_j < n}r^i_{j,s^i_j} + \hat{c}^i_{n-1,s^i_j}\} \nonumber \\
&\leq \l + \sum_{n=1}^{\infty} \sum_{s^i=1}^{n-1}\sum_{s^i_j = \l}^{n-1}\{r^{i,*}_{s^i} + \hat{c}^i_{n,s^i} \leq r^{i}_{j,s^i_j}+c^i_{n,s^i_j}\}
\end{align}
The following analysis applies for general $i$ and thus we omit the $i$ notation. Observe $r^*_s + \hat{c}_{n,s} \leq r_{j,s_j}+c_{n,s_j}$ implies that at least one of the following must hold,
\begin{align}
r^*_s &\leq \mu^*-\hat{c}_{n,s}, r_{j,s_j} \geq \mu_j + \hat{c}_{n,s_j} , \mu^* < \mu_j + 2\hat{c}_{n,s_j}
\end{align}
We bound each term.
\begin{align}
& P\{r^*_s \leq \mu^*-\hat{c}_{n,s}\}\leq P\{r^*_s \leq \mu^*-(\sqrt{2}-\alpha)\sqrt{\frac{\log n}{n}}\}\leq e^{-2(\sqrt{2}-\alpha)^2\log n} = n^{-2(\sqrt{2}-\alpha)^2}
\end{align}
the first inequality is due to the reason $n \geq s$ and $\beta_{j}(n) \geq 0$. And similarly
\begin{align}
P\{r_{j,s_j} &\geq \mu_j + \hat{c}_{n,s_j} \} \leq n^{-2(\sqrt{2}-\alpha)^2}
\end{align}
Let $\l = \lceil(2\sqrt{2}-\alpha[1-\beta_{j}(t)])^2 \log t/\Delta^2_j \rceil$
\begin{align}
&\mu^*-\mu_j-2\hat{c}_{n,s_j} \nonumber \\
&\geq \mu^*-\mu_j-(2\sqrt{2}-\alpha\cdot [1-\beta_{j}(n)])\sqrt{\log t/s_j} \nonumber \\
&\geq \mu^*-\mu_j-\Delta_j = 0
\end{align}
Therefore
\begin{align}
E[T_j(t)] &\leq \lceil\frac{(2(\sqrt{2}-\alpha\cdot[1-\beta_{j}(t)]))^2 \log t}{\Delta^2_j} \rceil \nonumber \\
&+ \sum_{n=1}^{\infty}\sum_{s=1}^{n-1}\sum_{s_j = \lceil(2\sqrt{2}-\alpha\cdot (1-\beta_{j}(t)))^2 \atop \log t/\Delta^2_j \rceil}^{t-1} (P\{r^*_s \leq \mu^*-\hat{c}_{n,s}\}+P\{r_{j,s_j} \geq \mu_j + \hat{c}_{n,s_j} \}) \nonumber \\
&\leq \lceil\frac{4(\sqrt{2}-\alpha \cdot [1-\beta_{j}(t)])^2 \log t}{\Delta^2_j} \rceil+\sum_{n=1}^{\infty}\sum_{s=1}^{n}\sum_{s_j = 1}^{n}2n^{-2(\sqrt{2}-\alpha)^2}
\end{align}
First notice the second term, if $-2(\sqrt{2}-\alpha)^2 < -3$, i.e.,
$\sqrt{2}-\sqrt{3/2} > \alpha$, the sum converges to a constant; i.e.,
\begin{align}
E[T_j(t)] \leq \lceil\frac{4(\sqrt{2}-\alpha\cdot [1-\beta_{j}(t)])^2 \log t}{\Delta^2_j} \rceil + \text{const.} \label{bound_1}
\end{align}
Next we try to bound $\beta_{j}(t), \forall j \in \overline{N}_K$. Remember that $\beta_{j}(t) = \frac{n_{j}(t)}{\sum_k n_{k}(t)}$ and we know
\begin{align}
n_{j}(t) \leq \lceil\frac{8 \log t}{\Delta^2_{j}} \rceil, \forall j \in \overline{N}_K
\end{align}
As there is no collision (due to decision sharing and collision avoidance), there are always $M\cdot K$ observations from the group. Thus for denominator we know
\begin{align}
\sum_k n_{k}(t) = M\cdot K\cdot n
\end{align}
Therefore
$
\beta_{j}(n) \leq \frac{M\cdot \lceil\frac{8 \log n}{\Delta^2_{j}} \rceil}{M\cdot K \cdot n} \rightarrow 0
$. The last approach is along with $n \rightarrow \infty$.  Therefore with a large enough time horizon $n$, all terms $\beta_{j}(t),  j\in \overline{N}_K$ goes to 0 and thus we have established the bounds as following
\begin{align}
R_G \leq \sum_{j \in \overline{N}_K}\lceil\frac{(6+\epsilon) \log t}{\Delta_{j}}\rceil + \text{const.} 
\end{align} 


\section{ Proof of Theorem \ref{dPartial}:}\label{dPart}

We analyze the probability that the number of chosen sub-optimal arms is higher than the optimal arms at time $t$. Consider $j \in \overline{N}_K$ and $* \in N_K$. 
\begin{align}
&P(\sum_{n=1}^t I_{d(n) = j} \geq \sum_{n=1}^t I_{d(n)=*})\nonumber \\
&=P(\sum_{n=1}^t I_{d(n) = j} \geq \sum_{n=1}^t I_{d(n)=*}|\sum_{n=1}^t I_{d(n)=*} \geq \mathcal O(t))\cdot P(\sum_{n=1}^t I_{d(n)=*} \geq \mathcal O(t)) \nonumber \\
&+P(\sum_{n=1}^t I_{d(n) = j} \geq \sum_{n=1}^t I_{d(n)=*}|\sum_{n=1}^t I_{d(n)=*} < \mathcal O(t))\cdot P(\sum_{n=1}^t I_{d(n)=*} < \mathcal O(t))
\end{align}
Consider the first term 
\begin{align}
&P(\sum_{n=1}^t I_{d(n) = j} \geq \sum_{n=1}^t I_{d(n)=*}|\sum_{n=1}^t I_{d(n)=*} \geq \mathcal O(t)) \leq P(\sum_{n=1}^t I_{d(n) = j} \geq \mathcal O(t)) \nonumber \\
&\leq \frac{E[\sum_{n=1}^t I_{d(n) = j} ]}{\mathcal O(t) }=\frac{E[T_j(t)]}{\mathcal O(t)} \leq \frac{\mathcal O(\log t)}{\mathcal O(t)}
\end{align}
Now consider $P(\sum_{n=1}^t I_{d(n)=*} < \mathcal O(t))$.
\begin{align}
&P(\sum_{n=1}^t I_{d(n)=*} < \mathcal O(t)) \leq P(\sum_{n=1}^t I_{d(n) \in \overline{N}_K} \geq \mathcal O(t))\nonumber \\
&\leq \frac{E[\sum_{n=1}^t I_{d(n)\in \overline{N}_K}]}{\mathcal O(t)} \leq \frac{\sum_{k \in \overline{N}_K}E[\sum_{n=1}^t I_{d(n)=k}]}{\mathcal O(t)} \leq \frac{\mathcal O(\log t)}{\mathcal O(t)}
\end{align}
Proved.

\end{document}